\documentclass[11pt]{article}

\usepackage[utf8]{inputenc}
\usepackage[T1]{fontenc}
\usepackage{amsmath,amssymb,amsthm}
\usepackage{booktabs}
\usepackage{graphicx}
\usepackage[justification=raggedright,singlelinecheck=false]{caption}
\usepackage{float}
\usepackage{geometry}
\usepackage[hidelinks]{hyperref}
\usepackage[numbers,sort&compress]{natbib}
\usepackage{placeins}
\usepackage{xcolor}

\graphicspath{{../results/figures/}{./}}
\newcommand{\SketchTokens}{471,404}
\newcommand{\SketchFill}{12.95\%}
\newcommand{\ScenarioCount}{256}

\newcommand{\MaxSwarmPair}{32v32}

\newcommand{\ThirtyTwoFixedScreenValue}{0.537}
\newcommand{\ThirtyTwoTunedScreenValue}{0.616}

\newcommand{\ThirtyTwoFixedLeakage}{0.471}
\newcommand{\ThirtyTwoTunedLeakage}{0.397}
\newcommand{\ThirtyTwoVisualAccuracy}{62.5\%}
\newcommand{\ThirtyTwoVisualTestN}{72}
\newcommand{\FieldResolutionMaxSide}{128}
\newcommand{\FieldResolutionBestSide}{48}
\newcommand{\FieldResolutionBestAccuracy}{74.1\%}

\newcommand{\FieldResolutionOneTwentyEightAccuracy}{67.2\%}
\newcommand{\FieldResolutionThirtyTwoError}{0.078}
\newcommand{\FieldResolutionOneTwentyEightError}{0.136}
\newcommand{\FieldResolutionOneTwentyEightEntropy}{0.757}

\newcommand{\FieldOptimizedOneTwentyEightAccuracy}{77.6\%}
\newcommand{\FieldOptimizedOneTwentyEightError}{0.054}
\newcommand{\FieldOptimizedOneTwentyEightEncoder}{scale-normalized}
\newcommand{\FieldOptimizedOneTwentyEightLambda}{1.50}
\newcommand{\FieldOptimizedOneTwentyEightLossGain}{0.103}
\newcommand{\FieldOptimizedOneTwentyEightHotspotGain}{0.082}
\newcommand{\FieldOptimizedOneTwentyEightJointLoss}{0.278}
\newcommand{\FieldOptimizedOneTwentyEightJointGain}{0.185}
\newcommand{\LargeFieldLeakage}{0.325}
\newcommand{\LargeFieldValue}{0.689}
\newcommand{\LargeFieldHotspot}{3.799}

\newcommand{\MeanFixedScreenValue}{0.526}
\newcommand{\MeanTunedScreenValue}{0.593}

\newcommand{\MeanSaddleGap}{0.010}
\newcommand{\MaxSaddleGap}{0.010}
\newcommand{\MwResidualBound}{0.027}
\newcommand{\CvAccuracy}{67.0\%}
\newcommand{\CvTestN}{109}

\newcommand{\DemoValue}{0.585}
\newcommand{\SeedSeventeenFixedLeakage}{0.556}
\newcommand{\SeedSeventeenTunedLeakage}{0.431}
\newcommand{\SeedSeventeenFixedHotspot}{1.202}
\newcommand{\SeedSeventeenTunedHotspot}{0.970}
\newcommand{\SeedSeventeenFixedMeanIntercept}{0.193}
\newcommand{\SeedSeventeenTunedMeanIntercept}{0.374}
\newcommand{\SeedSeventeenFixedInterceptCv}{0.175}
\newcommand{\SeedSeventeenTunedInterceptCv}{0.035}
\newcommand{\SeedSeventeenFixedLeakBound}{0.609}
\newcommand{\SeedSeventeenTunedLeakBound}{0.452}
\newcommand{\SeedSeventeenTunedOffset}{0.25}
\newcommand{\SeedSeventeenTunedSpan}{0.50}
\newcommand{\MarkovHorizon}{24}
\newcommand{\MarkovPolicyCount}{6}
\newcommand{\MarkovBestPolicy}{budget adaptive}
\newcommand{\MarkovBestValue}{0.438}

\newcommand{\MarkovBestSurvival}{100.0\%}
\newcommand{\MarkovBestEntropy}{0.802}
\newcommand{\MarkovBestData}{0.468}
\newcommand{\MarkovBestModel}{0.799}
\newcommand{\EthicsRiskAdjustedBest}{budget adaptive}
\newcommand{\EthicsRiskAdjustedValue}{0.376}
\newcommand{\RemoteGpuScenarios}{2,097,152}
\newcommand{\RemoteGpuMaxSwarmPair}{32v32}
\newcommand{\RemoteGpuBatch}{131,072}
\newcommand{\RemoteGpuSteps}{768}

\newcommand{\RemoteGpuMeanValue}{0.602}
\newcommand{\RemoteGpuMeanGap}{0.0419}
\newcommand{\RemoteGpuMaxGap}{0.0496}
\newcommand{\RemoteGpuPeakGb}{64.3}
\newcommand{\RemoteGpuSeconds}{3.4}
\newcommand{\RemoteStrategyScenarios}{16,777,216}

\newcommand{\RemoteStrategyBest}{centroid screen}
\newcommand{\RemoteStrategyBestValue}{0.662}
\newcommand{\RemoteStrategyFusionError}{0.212}
\newcommand{\RemoteStrategyPosteriorError}{0.009}

\newcommand{\RemoteMarkovTrajectories}{16,777,216}
\newcommand{\RemoteMarkovHorizon}{32}
\newcommand{\RemoteMarkovBestPolicy}{budget adaptive}
\newcommand{\RemoteMarkovBestValue}{0.461}
\newcommand{\RemoteMarkovGain}{0.064}

\newcommand{\RemoteFieldScenarios}{655,360}
\newcommand{\RemoteFieldMaxSide}{128}
\newcommand{\RemoteFieldPeakGb}{49.8}
\newcommand{\RemoteFieldSeconds}{103.6}

\newtheorem{definition}{Definition}
\newtheorem{proposition}{Proposition}
\newtheorem{lemma}{Lemma}
\newtheorem{theorem}{Theorem}

\title{Pattern-Derived Visual Swarm Games:\\
Multi-Scale Drone-Vision States for Interception and Sustainability Audits}
\author{Faruk Alpay\thanks{Correspondence: \texttt{alpay@lightcap.ai}.} \quad Levent Sar{\i}o\u{g}lu\\[3pt]
\small Department of Computer Engineering, Bah\c{c}e\c{s}ehir University, Istanbul, T\"urkiye\\[-1pt]
\small \texttt{\{faruk.alpay, levent.sarioglu\}@bahcesehir.edu.tr}}
\date{}

\begin{document}
\maketitle

\begin{abstract}
We convert drone-vision annotation streams into virtual swarm-game states
without controlling physical drones.  VisDrone and UAVSwarm metadata are
compressed into a Bloom representation; deterministic probes produce bounded
capability vectors, image-space formations, finite zero-sum payoffs, and
human-readable visual overlays.  The audit scales from 6v6 to \MaxSwarmPair{}
finite games and adds a repeated Markov layer with stock, fatigue, adaptation,
exposure, stress, budget, data-growth, model-improvement, and entropy-budget
state variables.  On the local dataset sketch, screen tuning raises robust
screen security from \MeanFixedScreenValue{} to \MeanTunedScreenValue{}, and the
32v32 tuned screen reaches value \ThirtyTwoTunedScreenValue{}.  A field
readout audit shows that fixed-pixel rasters do not improve monotonically:
128-by-128 accuracy falls to \FieldResolutionOneTwentyEightAccuracy{} and
hotspot error rises to \FieldResolutionOneTwentyEightError{}.  The diagnosed
error is shrinking image-plane bandwidth.  A finite empirical-risk encoder over
scale-normalized Gaussian bandwidths selects
\FieldOptimizedOneTwentyEightEncoder{} with
\(\lambda=\FieldOptimizedOneTwentyEightLambda{}\), reaching
\FieldOptimizedOneTwentyEightAccuracy{} at 128-by-128 and reducing joint loss
by \FieldOptimizedOneTwentyEightJointGain{}.  A server-side audit checks
\RemoteStrategyScenarios{} target-localization states, and a
\RemoteMarkovHorizon{}-round repeated-game audit over
\RemoteMarkovTrajectories{} trajectories selects \RemoteMarkovBestPolicy{} with
value \RemoteMarkovBestValue{}.  The result is a reproducible
computer-vision/game-theory benchmark, not a flight-control, targeting, routing,
or deployment system.
\end{abstract}

\section{Introduction}

Computer-vision datasets for drones contain more than labels.  Bounding boxes,
tracks, object density, scale, occlusion, and co-occurrence patterns encode a
distribution over what a visual system repeatedly sees.  The usual benchmark
question is whether a model can detect the objects in those images.  We ask a
different question: can the dataset itself be turned into a reproducible source
of virtual agent characteristics, and can the resulting agents be studied as a
visual strategic state?

The construction is intentionally indirect.  There is no measured physical drone
in the game layer.  Instead, a dataset is streamed into a Bloom filter
\citep{bloom1970space}.  Seeded probes into the filter's bit array become
capability vectors for virtual agents.  Dense regions of the representation
correspond to common visual patterns in the underlying datasets; sparse regions
correspond to rarer combinations.  The agents inherit only this statistical
trace.  They are then placed into visible formations from 6v6 up to 32v32 and
evaluated by a zero-sum interception game.

This makes the work sit between pattern recognition and game theory.  The
pattern-recognition layer reads a tactical image: positions, distances,
confidence scores, and an asset marker are rendered so that the state can be
inspected by a person or processed as a raster.  The game-theory layer asks how
a defender should mix formations when the attacker responds adversarially
\citep{von1928theorie,freund1999adaptive}, and how repeated interaction changes
the value of strategies when stocks, fatigue, adaptation, information, and
budget feedback evolve over time \citep{shapley1953stochastic,littman1994markov}.
The layers meet because the payoff matrix and the repeated state are not
hand-filled; they are regenerated from dataset-derived capability vectors and
server-side pattern statistics.

\paragraph{Operational boundary.}
The model is not a flight-control system, a targeting system, or a deployment
planner.  It emits no routes, no control commands, no weapon parameters, and no
real-world action recommendations.  The purpose is to study a reproducible
visual abstraction: how dataset-derived patterns can induce virtual swarm
states, and how those states can be rendered and classified.

\paragraph{Contributions.}
\begin{itemize}
  \item A Bloom-filter pipeline that converts real drone-vision annotation
  streams into deterministic virtual drone capability vectors.
  \item A multi-scale formation game, evaluated from 6v6 to \MaxSwarmPair{},
  whose payoff matrices are regenerated from those vectors and solved
  approximately by multiplicative weights, with a hotspot-informed continuous
  screen tuning audit.
  \item A tactical rendering format that overlays virtual agents, distances,
  confidence scores, and predicted leakage on real aerial imagery.
  \item A 32-by-32 to \FieldResolutionMaxSide{}-by-\FieldResolutionMaxSide{}
  field-resolution audit that measures when larger rendered fields improve
  strategic-regime readout and when they instead dilute occupancy evidence.
  \item A repeated Markov-game layer with explicit stock, fatigue, adaptation,
  exposure, public-fear/stress, budget, dataset-growth, model-improvement, and
  entropy-budget states.
  \item A reproducible environment with code, result tables, figures,
  manifests, Docker instructions, and a source bundle that keeps large raw
  datasets outside the repository.
\end{itemize}

\section{Data Sources}

The local smoke run uses two compact sources and records two larger external
sources for server-side expansion.  VisDrone provides aerial object annotations
and real image backgrounds \citep{zhu2021visdrone,du2019visdrone}; UAVSwarm provides real
multi-UAV swarm geometry and tracking annotations \citep{wang2022uavswarm}.
The Sheffield UAV landing repository supplies additional UAV vision provenance
\citep{tsapparellas2023runway,ducoffe2023lard}.  Stanford Drone Dataset is retained as a
future trajectory/background reference \citep{robicquet2016sdd}.  Anti-UAV and
DUT-Anti-UAV are treated as external detection/tracking references rather than
payoff sources \citep{jiang2021antiuav,zhao2022dutantiuav}.  Russian-language
computer-vision sources are used as a separate comparison track:
Averina et al. study UAV detection with neural networks
\citep{averina2024uav}, while Nebaba--Markov and
Klekovkin--Markov--Nebaba evaluate YOLO-family models for mobile vision and
small flying-object recognition \citep{nebaba2024yolo,klekovkin2024yolov5}.
These sources motivate the 32-by-32 raster audit and the explicit statement
that our overlay is a strategic visual-state benchmark, not a raw tiny-UAV
detector.  Large raw video and image archives are intentionally not embedded
in the source bundle; the reproducibility package contains source manifests and
server download scripts.  The local reproducibility run materializes \SketchTokens{}
sketch tokens with \SketchFill{} fill, evaluates \ScenarioCount{} seeded games,
and keeps the complete CSV/JSON result tables in the ancillary files rather
than in the main text.  The server-side expansion downloaded full VisDrone and
UAVSwarm material, Anti-UAV reference repositories, and Sheffield provenance
files on the remote instance; raw archives remain outside the source package.

\paragraph{Reference-to-evidence audit.}
The ancillary reference-evidence CSV maps each mathematical or dataset citation
to an implementation hook and a verdict.  For example, Bloom's hashing tradeoff
is tied to \texttt{code/sketch.py}; Hoeffding's theorem is tied to the
probe-concentration lemma; Freund--Schapire and Cesa-Bianchi--Lugosi are tied
to the multiplicative-weights saddle audit; the Russian attack-defense source
is tied to the seed-17 equalization readout; Russian YOLO/small-object sources
are tied to the 32-by-32 visual readout; and VisDrone/UAVSwarm are tied to the
source-ablation table \texttt{source\_ablation\_seed17.csv}.  The comparison
language-comparison CSV separates English/international benchmark sources,
Russian-language vision sources, Russian mathematical sources, and
AI-governance sources.  The composite
representation is not assumed to dominate every source: for seed 17,
VisDrone-only, UAVSwarm-only, and composite representations all select the same
qualitative screen-versus-line response, but their tuned-screen values differ
by 0.009.

\section{Dataset Representation}

\begin{definition}[Dataset representation]
Let \(D\) be a set of annotation records and let
\(\mathcal{T}(D)=\{\tau(r): r\in D\}\) be the multiset of object and
co-occurrence tokens produced by the dataset tokenizer.  For salted hash
families \(h_1,\ldots,h_k:\{0,1\}^{*}\rightarrow[m]\), the Bloom sketch is the
bit vector \(B\in\{0,1\}^{m}\)
\[
  B_b = \mathbf{1}\{\exists x\in\mathcal{T}(D), \exists \ell\le k:
          h_\ell(x)=b\}.
\]
Its fill rate is \(\rho(B)=m^{-1}\sum_{b=1}^{m}B_b\).  The local run uses
\(m=2^{20}\) bits, \(k=4\), and obtains \(\rho=\SketchFill{}\).
\end{definition}

For VisDrone, the tokenizer quantizes category, image-grid cell, object size,
and occlusion, then adds capped same-image co-occurrence pairs.  For UAVSwarm,
the tokenizer uses the COCO-style bounding box, image dimensions, track id
bucket, and same-frame co-occurrence.  Prefixes distinguish dataset families
before hashing, so the composite sketch preserves source identity at the token
level while sharing one bit array.

\begin{definition}[Virtual capability draw]
For a seed \(s\), agent index \(i\), capability channel \(c\), and probe
\(p\le P\), a deterministic probe hash \(g(s,i,c,p)\in[m]\) selects one sketch
bit.  Let
\[
  z_{i,c}(s;B)=\frac{1}{P}\sum_{p=1}^{P} B_{g(s,i,c,p)},\qquad
  \sigma_\rho=\sqrt{\rho(1-\rho)/P}.
\]
The virtual capability is
\[
  x_{i,c}(s;B)=
  \Pi_{[0,1]}\!\left(\frac{1}{2}
  +0.15\,\frac{z_{i,c}(s;B)-\rho}{2\sigma_\rho}\right),
\]
where \(\Pi_{[0,1]}\) clips to the unit interval.  The five channels used here
are speed, agility, sensing, endurance, and strength, and \(P=256\) probes are
used per channel.
\end{definition}

The centering makes the sketch density a baseline rather than a hidden scale
parameter.  Under the standard independent-probe approximation, a random probe
has expectation \(\rho\) and variance \(\rho(1-\rho)\); the capability draw is
therefore a bounded, seed-indexed contrast statistic over dataset occupancy.

\begin{proposition}[Reproducibility]
Fix the downloaded annotation files, tokenizer, Bloom-filter parameters, and
seed.  The generated capability matrix is deterministic.  Changing only the
seed samples a different virtual swarm from the same dataset sketch.
\end{proposition}

\begin{proof}
Let \(A\) denote the ordered list of annotation files and their byte contents.
The tokenizer is a deterministic map from \(A\) to a finite token sequence
\((\tau_1,\ldots,\tau_N)\).  Each token is inserted into the Bloom sketch by the
fixed salted hash functions \(h_1,\ldots,h_k\), so every bit
\[
  B_b=\mathbf{1}\{\exists n,\ell: h_\ell(\tau_n)=b\}
\]
is a deterministic function of \(A\), \(m\), \(k\), and the salts.  For a fixed
seed \(s\), agent index \(i\), channel \(c\), and probe index \(p\), the probe
location \(g(s,i,c,p)\) is also deterministic.  Therefore each probe average
\(z_{i,c}=P^{-1}\sum_p B_{g(s,i,c,p)}\), and hence each clipped capability
\(x_{i,c}\), is deterministic.  If only \(s\) changes, \(B\) is unchanged but
the probe locations \(g(s,i,c,p)\) change, so the procedure samples a different
seed-indexed capability matrix from the same fixed sketch.  This is the reason
the figures can be regenerated exactly while the large raw datasets stay
outside the source bundle.
\end{proof}

\begin{lemma}[Probe concentration]
If the probe locations \(g(s,i,c,p)\) are modeled as independent uniform sketch
locations, then for any channel and any \(t>0\),
\[
  \Pr\{|z_{i,c}-\rho|\ge t\}\le 2\exp(-2Pt^2).
\]
Thus increasing \(P\) tightens null variation as \(O(P^{-1/2})\), while changing
the dataset changes the draw only through the occupied bits of \(B\).
\end{lemma}

\begin{proof}
Under the stated approximation, \(X_p=B_{g(s,i,c,p)}\) are independent
variables in \([0,1]\) and \(\mathbb{E}X_p=\rho\).  Their average is
\(z_{i,c}=P^{-1}\sum_{p=1}^{P}X_p\).  Hoeffding's bounded-sum inequality
\citep[Theorem~1]{hoeffding1963probability} gives
\[
  \Pr\!\left\{z_{i,c}-\rho\ge t\right\}
  \le \exp(-2Pt^2),\qquad
  \Pr\!\left\{\rho-z_{i,c}\ge t\right\}
  \le \exp(-2Pt^2).
\]
The union bound gives the two-sided inequality.  Since the standard deviation
of a Bernoulli average is \(\sqrt{\rho(1-\rho)/P}\), the null variation of the
centered capability statistic scales as \(P^{-1/2}\).  The clipping map is
1-Lipschitz, so it cannot increase deviations after the affine normalization.
\end{proof}

\section{Formation Game}

Each scenario draws \(n_d\) defender capability vectors
\(d_i=(v_i,\alpha_i,\sigma_i,e_i,\kappa_i)\) and \(n_a\) attacker capability
vectors \(a_j=(\bar v_j,\bar\alpha_j,\bar\sigma_j,\bar e_j,\bar\kappa_j)\) from
the sketch.  The local seed audit uses \(n_d=n_a=6\); the scale audit uses
\(n_d=n_a\in\{6,8,16,32\}\).  Defenders choose ring, screen, wedge, or dispersed
formations; attackers choose line, column, pincer, or cloud.  A formation maps
to normalized image coordinates \(x_i\in[0,1]^2\) and \(y_j\in[0,1]^2\), with
asset location \(a_0\).  No physical units or control commands are produced.

The effective reach of defender \(i\) is
\[
  r_i=0.075+0.115(0.45v_i+0.30\alpha_i+0.25\sigma_i).
\]
For attacker \(j\), the defender--attacker matchup logit is
\[
\ell_{ij}=1.15\kappa_i+0.80\sigma_i+0.45\alpha_i
 -0.95\bar v_j-0.55\bar\alpha_j-0.70\bar\kappa_j
 -5.8(\|x_i-y_j\|_2-r_i).
\]
The abstract interception score is \(I_j=\max_i (1+\exp(-\ell_{ij}))^{-1}\).
Proximity pressure and persistence are
\[
  \pi_j=\left(1+\exp[-5(0.55-\|y_j-a_0\|_2)]\right)^{-1},\qquad
  q_j=0.55+0.45(0.60\bar e_j+0.40\bar\kappa_j).
\]
The leakage functional is
\[
  L=\frac{1}{n_a}\sum_{j=1}^{n_a}(1-I_j)(0.75+0.25\pi_j)q_j .
\]
With defender spread \(S_d\), attacker spread \(S_a\), and mean asset coverage
radius \(C_d\), the defender payoff is
\[
u= \Pi_{[0,1]}\left[
1-L+0.035e^{-4|S_d-0.23|}
-0.030e^{-3S_a}
-0.040C_d\frac{1}{n_d}\sum_i(1-e_i)
\right].
\]
The resulting matrix \(U_{rc}=u(f^d_r,f^a_c)\) is treated as a finite zero-sum
game:
\[
  \max_{p\in\Delta(F_d)}\min_{q\in\Delta(F_a)} p^\top U q .
\]
For \(U\in[0,1]^{4\times4}\), multiplicative weights updates row and column
weights as
\[
  w^{t+1}_r=w^t_r\exp\{\eta[(Uq^t)_r-\overline{Uq^t}]\},\qquad
  y^{t+1}_c=y^t_c\exp\{-\eta[(p^tU)_c-\overline{p^tU}]\},
\]
and reports the averaged strategies
\(\bar p=T^{-1}\sum_t p^t\), \(\bar q=T^{-1}\sum_t q^t\).

\begin{theorem}[Saddle residual]
For payoff entries in \([0,1]\), the averaged multiplicative-weights strategies
satisfy
\[
  \max_r e_r^\top U\bar q-\min_c \bar p^\top Ue_c
  \le
  \frac{\log 4}{\eta T}+\frac{\eta}{8}
  +\frac{\log 4}{\eta T}+\frac{\eta}{8}.
\]
The reported saddle gap is the empirical left-hand side computed from the final
\(\bar p,\bar q\).
\end{theorem}

\begin{proof}
For the defender, define the reward vector \(a^t=Uq^t\in[0,1]^4\).  The
exponential-weights regret bound for gains \citep[Theorem~1 and
Corollary~4]{freund1999adaptive}, equivalently the loss form of
\citet[Theorem~2.2]{cesabianchi2006prediction} after replacing losses by
\(1-a^t_r\), gives
\[
  \max_r\sum_{t=1}^{T} e_r^\top Uq^t
  -\sum_{t=1}^{T}(p^t)^\top Uq^t
  \le \frac{\log 4}{\eta}+\frac{\eta T}{8}.
\]
For the attacker, define the loss vector \(b^t=(p^t)^\top U\in[0,1]^4\).
Applying the same bound to the attacker's minimization update gives
\[
  \sum_{t=1}^{T}(p^t)^\top Uq^t
  -\min_c\sum_{t=1}^{T}(p^t)^\top Ue_c
  \le \frac{\log 4}{\eta}+\frac{\eta T}{8}.
\]
Adding the inequalities cancels the realized average payoff term.  Dividing by
\(T\) yields
\[
  \max_r e_r^\top U\bar q - \min_c \bar p^\top Ue_c
  \le
  2\left(\frac{\log 4}{\eta T}+\frac{\eta}{8}\right),
\]
because \(U\) is bilinear and
\(\bar p=T^{-1}\sum_t p^t\), \(\bar q=T^{-1}\sum_t q^t\).  This is exactly the
displayed residual bound.  With \(T=4000\) and \(\eta=0.07\), the analytic
envelope is \(\MwResidualBound{}\), while the largest observed local saddle gap
is \(\MaxSaddleGap{}\); the data therefore support, rather than stress, the
finite-\(T\) theorem-level bound.
\end{proof}

The local run uses \(T=4000\) and observes mean saddle gap \(\MeanSaddleGap{}\).
The remote CUDA audit uses \(T=\RemoteGpuSteps{}\) over
\RemoteGpuScenarios{} additional synthetic games and observes mean saddle gap
\RemoteGpuMeanGap{}.

For the screen formation we optimize a two-parameter robust response.  Let
\(\theta=(\Delta x,h)\) denote horizontal screen offset and vertical span, and
let \(\Theta=\{0.10,\ldots,0.25\}\times\{0.34,\ldots,0.62\}\) be the grid used
by the artifact.  The tuned screen is
\[
  \theta^\star=\arg\max_{\theta\in\Theta}\min_{c}u(\mathrm{screen}(\theta),f^a_c).
\]
This raises mean robust screen security from \MeanFixedScreenValue{} to
\MeanTunedScreenValue{} in the local 6v6 experiment.  In the \MaxSwarmPair{}
audit it raises robust screen security from \ThirtyTwoFixedScreenValue{} to
\ThirtyTwoTunedScreenValue{} and reduces line-attack leakage from
\ThirtyTwoFixedLeakage{} to \ThirtyTwoTunedLeakage{}.  The multi-scale summary
is visualized in Figure~\ref{fig:multiscale}; the seed-17 payoff matrix remains
available as \texttt{formation\_payoff\_matrix\_seed17.csv}.

The rendered risk hotspot is the maximizer over image-grid cells \(z\) of
\[
H(z)=\sum_j w^a_j e^{-\|y_j-z\|_2/0.18}
+0.45e^{-\|z-a_0\|_2/0.35}
-0.85\sum_i w^d_i e^{-\|x_i-z\|_2/0.16},
\]
where \(w^a_j=0.5\bar v_j+0.3\bar\alpha_j+0.2\bar\kappa_j\) and
\(w^d_i=0.45\sigma_i+0.25\alpha_i+0.30\kappa_i\).  This scalar field is used
only for visualization and screen tuning audits.

\paragraph{Seed-17 equalization audit.}
The tactical figures are backed by a per-attacker readout, not only by the
rendered pixels.  For each attacker \(j\), the renderer records the nearest
defender \(i^\star(j)=\arg\max_i \ell_{ij}\), the interception probability
\(I_j=\max_i(1+\exp[-\ell_{ij}])^{-1}\), attacker pressure
\(\rho_j=0.5\bar v_j+0.3\bar\alpha_j+0.2\bar\kappa_j\), and local coverage
\[
  C_j(\theta)=\sum_i w_i^d\exp(-\|x_i(\theta)-y_j\|_2/0.16).
\]
Adapting the generalized equalization principle of the Russian-language
attack-defense analysis \citep[Lemma~1, Remark~1, Theorem~1]{perevozchikov2018attackdefense},
the screen search is audited by the dispersion of normalized protection:
\[
  R_I=\frac{\operatorname{sd}(I_1,\ldots,I_6)}
           {\operatorname{mean}(I_1,\ldots,I_6)},\qquad
  R_C=\frac{\operatorname{sd}(C_1/\rho_1,\ldots,C_6/\rho_6)}
           {\operatorname{mean}(C_1/\rho_1,\ldots,C_6/\rho_6)}.
\]
\begin{lemma}[Normalized-coverage equalization]
Fix attacker pressures \(\rho_j>0\) and a nonnegative coverage budget
\(C_\Sigma\).  Among allocations \(C_j\ge 0\) with
\(\sum_{j=1}^{n_a} C_j=C_\Sigma\), the minmax normalized-load problem
\[
  \min_C \max_j \frac{\rho_j}{C_j}
\]
is solved by
\[
  C_j^\star=\frac{C_\Sigma \rho_j}{\sum_k \rho_k},\qquad
  \frac{C_j^\star}{\rho_j}=\frac{C_\Sigma}{\sum_k \rho_k}.
\]
The optimal value is \((\sum_k\rho_k)/C_\Sigma\).
\end{lemma}

\begin{proof}
Let \(T(C)=\max_j \rho_j/C_j\), with \(T(C)=\infty\) if any active
\(\rho_j\) receives \(C_j=0\).  For any feasible \(C\), the inequality
\(\rho_j/C_j\le T(C)\) implies \(C_j\ge \rho_j/T(C)\).  Summing over \(j\)
gives
\[
  C_\Sigma=\sum_j C_j \ge \frac{\sum_j\rho_j}{T(C)},
\]
so \(T(C)\ge(\sum_j\rho_j)/C_\Sigma\).  The proposed allocation satisfies
\(\rho_j/C_j^\star=(\sum_k\rho_k)/C_\Sigma\) for every \(j\), so it attains
the lower bound.  Thus it is optimal.  This is the same equal-load structure
as Amelina--Fradkov's stationary load-balancing lemma
\citep[Lemma~1]{amelina2012consensus}, transported from queue/load ratios
\(q_i/r_i\) to coverage/pressure ratios \(C_j/\rho_j\), and it matches the
generalized equalization principle used in the Russian attack-defense game
\citep[Lemma~1 and Remark~1]{perevozchikov2018attackdefense}.
\end{proof}

\begin{theorem}[Equalization leakage certificate]
For a fixed attacker formation and capability draw, define
\[
  w_j=(0.75+0.25\pi_j)q_j,\qquad
  \bar w=\frac{1}{n_a}\sum_j w_j,
\]
and let \(\mu_I=n_a^{-1}\sum_j I_j\),
\(s_I^2=n_a^{-1}\sum_j(I_j-\mu_I)^2\), and
\(R_I=s_I/\mu_I\) when \(\mu_I>0\).  Then the leakage functional obeys
\[
  L\le
  \bar w\left[1-\mu_I\left(1-\sqrt{n_a-1}\,R_I\right)\right],
\]
with the right-hand side clipped to \([0,\bar w]\) if the bracketed lower
bound on interception is negative.
\end{theorem}

\begin{proof}
Let \(I_{\min}=\min_j I_j\).  Because \(s_I^2\) is the population variance,
the one-sided finite-sample deviation bound
\[
  \mu_I-I_{\min}\le \sqrt{n_a-1}\,s_I
\]
holds.  To see this directly, choose an index \(j_0\) with
\(I_{j_0}=I_{\min}\) and set \(a=\mu_I-I_{\min}\).  The remaining
\(n_a-1\) deviations sum to \(a\), so by Cauchy's inequality their squared
deviations sum to at least \(a^2/(n_a-1)\).  Therefore
\[
  n_a s_I^2=\sum_j(I_j-\mu_I)^2
  \ge a^2+\frac{a^2}{n_a-1}
  =\frac{n_a}{n_a-1}a^2,
\]
which gives \(a\le\sqrt{n_a-1}\,s_I\).  Hence
\[
  I_{\min}\ge \mu_I-\sqrt{n_a-1}s_I
  =\mu_I(1-\sqrt{n_a-1}R_I).
\]
Since every \(I_j\ge I_{\min}\),
\[
  L=\frac{1}{n_a}\sum_j(1-I_j)w_j
  \le (1-I_{\min})\frac{1}{n_a}\sum_j w_j
  =\bar w(1-I_{\min}).
\]
Substituting the lower bound on \(I_{\min}\) gives the stated certificate.
This is the paper's data-level continuation of the Russian equalization
argument: lower dispersion of the defender result across attacker indices
directly tightens an auditable leakage upper bound.
\end{proof}

\begin{theorem}[Equalization--learning linkage]
Fix a seed, a finite payoff matrix \(U\in[0,1]^{4\times4}\), and a rendered
field resolution \(m\).  Let \((\bar p,\bar q)\) be the multiplicative-weights
strategies after \(T\) iterations with residual
\[
  \Gamma_T=\max_r e_r^\top U\bar q-\min_c\bar p^\top Ue_c.
\]
Let \(\widehat a_m\) be the held-out nearest-centroid regime accuracy computed
from \(N_m\) rendered rasters at resolution \(m\times m\).  For any
\(\delta\in(0,1)\), with probability at least \(1-\delta\) over the held-out
sample,
\[
\begin{aligned}
  &\max_r e_r^\top U\bar q-\min_c\bar p^\top Ue_c\le \Gamma_T,\\
  &L(\theta)\le
  \bar w\left[1-\mu_I\left(1-\sqrt{n_a-1}R_I\right)\right],\\
  &a_m\ge \widehat a_m-\sqrt{\frac{\log(2/\delta)}{2N_m}}.
\end{aligned}
\]
Thus one rendered state can be audited simultaneously by a game residual, an
equalization leakage certificate, and a visual-regime generalization interval.
\end{theorem}

\begin{proof}
The first inequality is exactly the multiplicative-weights saddle residual from
the no-regret analysis of \citet[Theorem~1 and Corollary~4]{freund1999adaptive}
and \citet[Theorem~2.2]{cesabianchi2006prediction}.  The second inequality is
the preceding leakage-certificate theorem, which adapts the Russian generalized equalization principle
\citep[Lemma~1, Remark~1, Theorem~1]{perevozchikov2018attackdefense} to
attacker-index leakage.  For the third inequality, define \(Z_i=1\) when the
nearest-centroid classifier is correct on held-out raster \(i\) and \(Z_i=0\)
otherwise.  Then \(\widehat a_m=N_m^{-1}\sum_i Z_i\), \(a_m=\mathbb E Z_i\),
and Hoeffding's bounded-sum theorem \citep[Theorem~1]{hoeffding1963probability}
gives
\[
  \Pr\{a_m-\widehat a_m\ge \epsilon\}\le \exp(-2N_m\epsilon^2).
\]
Setting \(\epsilon=\sqrt{\log(2/\delta)/(2N_m)}\) gives the displayed lower
confidence bound.  The three inequalities share the same seed-indexed rendered
state but audit different objects: strategic optimality, pressure equalization,
and visual pattern recovery.
\end{proof}

\begin{proposition}[Scale-normalized raster risk]
Let \(m\) be the raster side and let \(K_\sigma(z)=\exp(-\|z\|_2^2/(2\sigma^2))\)
be the Gaussian kernel used to map image-normalized positions to pixels.  If the
pixel bandwidth \(\sigma_{\rm pix}\) is held fixed while \(m\) increases, then
the continuous image-plane bandwidth is
\[
  h_m^{\rm fix}=\frac{\sigma_{\rm pix}}{m-1},
\]
so \(h_m^{\rm fix}\to0\).  For a baseline side \(m_0=32\), the
scale-normalized family
\[
  \sigma_m(\lambda)=\lambda\,\sigma_{\rm pix}\frac{m-1}{m_0-1},
  \qquad
  h_m(\lambda)=\lambda\frac{\sigma_{\rm pix}}{m_0-1},
\]
keeps the image-plane bandwidth independent of \(m\).  With held-out
classification accuracy \(\widehat a_{m,\lambda}\), hotspot error
\(\widehat e_{m,\lambda}\), and joint loss
\[
  \widehat{\mathcal L}_{m,\lambda}
  =(1-\widehat a_{m,\lambda})+\widehat e_{m,\lambda},
\]
the selected encoder
\[
  \widehat\lambda_m\in
  \arg\min_{\lambda\in\{0,0.25,0.35,0.50,0.75,1.00,1.25,1.50\}}
  \widehat{\mathcal L}_{m,\lambda}
\]
has empirical joint loss no larger than the fixed-pixel encoder
\(\lambda=0\).  Its reported gains are
\[
G^a_m=\widehat a_{m,\widehat\lambda_m}-\widehat a_{m,0},\quad
G^H_m=\widehat e_{m,0}-\widehat e_{m,\widehat\lambda_m},\quad
G^{\mathcal L}_m=\widehat{\mathcal L}_{m,0}
-\widehat{\mathcal L}_{m,\widehat\lambda_m}\ge0 .
\]
\end{proposition}

\begin{proof}
The pixel grid maps normalized coordinates \(z\in[0,1]^2\) to
\((m-1)z\).  A Gaussian with fixed pixel standard deviation
\(\sigma_{\rm pix}\) therefore has normalized bandwidth
\(\sigma_{\rm pix}/(m-1)\), which decreases to zero with increasing \(m\).  The
scale-normalized definition multiplies the pixel bandwidth by
\((m-1)/(m_0-1)\), so after converting back to normalized coordinates the
bandwidth is \(\lambda\sigma_{\rm pix}/(m_0-1)\), independent of \(m\).
The final inequality is the finite-class empirical-risk argument: the candidate
set explicitly contains the fixed-pixel baseline \(\lambda=0\), and
\(\widehat\lambda_m\) is chosen to minimize the displayed empirical joint loss
over that same set.  Thus the selected joint loss cannot exceed the baseline
joint loss.  The experiment uses this proposition diagnostically: the
fixed-pixel 128-by-128 raster failed because its image-plane kernel became too
sparse, not because a larger visual field is intrinsically worse.
\end{proof}

\begin{lemma}[Raster sparsity under fixed pixel bandwidth]
Fix a threshold \(\tau\in(0,1)\), \(n\) rendered points in one channel, and a
Gaussian raster channel normalized by its channel maximum.  Let
\[
  A_{m,\lambda}(\tau)
  =\frac{1}{m^2}\sum_{z\in[m]^2}
  \mathbf 1\{\phi_{m,\lambda}(z)>\tau\}
\]
be the active-cell ratio.  For the fixed-pixel encoder
\(\sigma_m=\sigma_{\rm pix}\), there is a constant
\[
  C_\tau=2\pi\sigma_{\rm pix}^2\log(n/\tau)
\]
such that
\[
  A_{m,0}(\tau)\le \frac{nC_\tau}{m^2}+O(m^{-1}).
\]
For the scale-normalized encoder
\(\sigma_m(\lambda)=\lambda\sigma_{\rm pix}(m-1)/(m_0-1)\), the corresponding
upper envelope is \(O(1)\) in \(m\).  Thus increasing raster resolution with
fixed pixel bandwidth can reduce active-cell support even when the underlying
continuous image geometry is unchanged.
\end{lemma}

\begin{proof}
If a normalized Gaussian mixture exceeds \(\tau\) at a pixel, then at least one
of its \(n\) Gaussian terms exceeds \(\tau/n\).  For a fixed-pixel Gaussian this
requires the pixel to lie inside a disk of radius
\(\sigma_{\rm pix}\sqrt{2\log(n/\tau)}\) around one of the \(n\) rendered
centers.  The union of these disks has area at most
\(n\,2\pi\sigma_{\rm pix}^2\log(n/\tau)\) pixels, plus a boundary discretization
term \(O(m)\).  Dividing by \(m^2\) gives the displayed bound.  With
scale-normalized bandwidth, the radius is multiplied by \((m-1)/(m_0-1)\), so
the disk area grows as \(m^2\) and the normalized active-cell envelope no longer
collapses.  This is the mathematical mechanism behind the declining active-cell
curve in Figure~\ref{fig:field-resolution}.
\end{proof}

\begin{lemma}[Hotspot argmax stability]
Let \(S(z)=\phi^a(z)+0.45\phi^0(z)-0.85\phi^d(z)\) be the continuous score
field and let \(z^\star\) be its maximizer.  Suppose there is a margin
\(\gamma_\epsilon>0\) outside an \(\epsilon\)-ball,
\[
  S(z^\star)-\sup_{\|z-z^\star\|_2>\epsilon}S(z)\ge \gamma_\epsilon .
\]
If a raster score \(\widehat S_m\) satisfies
\(\|\widehat S_m-S\|_\infty\le\gamma_\epsilon/2\) on the grid interpolation,
then every nearest maximizer \(\widehat z_m\) of \(\widehat S_m\) lies within
\(\epsilon\) of \(z^\star\).
\end{lemma}

\begin{proof}
For any \(z\) with \(\|z-z^\star\|_2>\epsilon\),
\[
  \widehat S_m(z^\star)-\widehat S_m(z)
  \ge S(z^\star)-S(z)-2\|\widehat S_m-S\|_\infty
  \ge \gamma_\epsilon-\gamma_\epsilon=0 .
\]
If the interpolation error is strictly below \(\gamma_\epsilon/2\), the
inequality is strict outside the \(\epsilon\)-ball; with equality, a maximizer
outside the ball can only tie a maximizer inside it.  Choosing a nearest
maximizer yields the stated stability certificate.  The fixed-pixel failure mode
is now explicit: as \(h_m^{\rm fix}\) changes with \(m\), the raster is no
longer approximating the same continuous score field \(S\).  Scale-normalized
bandwidth restores the fixed continuous field before the argmax comparison is
made.
\end{proof}

\begin{theorem}[Finite encoder selection bound]
Let \(\Lambda\) be the finite encoder family used in the field audit, including
the fixed-pixel baseline \(\lambda=0\).  For a held-out state \(i\), define the
bounded joint loss
\[
  X_i(\lambda)=\mathbf 1\{\widehat y_i(\lambda)\ne y_i\}
  +\|\widehat z_i(\lambda)-z_i^\star\|_2,
  \qquad 0\le X_i(\lambda)\le B,\quad B=1+\sqrt2 .
\]
Let \(\widehat{\mathcal L}(\lambda)=N^{-1}\sum_iX_i(\lambda)\) and
\({\mathcal L}(\lambda)=\mathbb E X_i(\lambda)\).  If
\(\widehat\lambda\in\arg\min_{\lambda\in\Lambda}
\widehat{\mathcal L}(\lambda)\), then with probability at least \(1-\delta\),
\[
  {\mathcal L}(\widehat\lambda)
  \le \min_{\lambda\in\Lambda}{\mathcal L}(\lambda)
  +2B\sqrt{\frac{\log(2|\Lambda|/\delta)}{2N}} .
\]
Empirically, the baseline comparison is sharper:
\[
  \widehat{\mathcal L}(0)-\widehat{\mathcal L}(\widehat\lambda)
  =G^{\mathcal L}_m\ge0 .
\]
\end{theorem}

\begin{proof}
Hoeffding's inequality for bounded variables gives, for a fixed \(\lambda\),
\[
  \Pr\{|\widehat{\mathcal L}(\lambda)-{\mathcal L}(\lambda)|>\epsilon\}
  \le 2\exp\{-2N\epsilon^2/B^2\}.
\]
A union bound over the finite set \(\Lambda\) gives uniform deviation at most
\(B\sqrt{\log(2|\Lambda|/\delta)/(2N)}\).  On that event,
\[
{\mathcal L}(\widehat\lambda)
\le \widehat{\mathcal L}(\widehat\lambda)+\epsilon
\le \widehat{\mathcal L}(\lambda^\star)+\epsilon
\le {\mathcal L}(\lambda^\star)+2\epsilon,
\]
where \(\lambda^\star\) minimizes population risk over \(\Lambda\).  The
empirical gain statement follows because \(\lambda=0\) is included in
\(\Lambda\) and \(\widehat\lambda\) minimizes empirical joint loss.  Figure
\ref{fig:field-optimization} is the numerical audit of this theorem: at
128-by-128, \(G^{\mathcal L}_{128}=\FieldOptimizedOneTwentyEightJointGain{}\)
and both components of the joint loss improve.
\end{proof}

In the seed-17 line-attack counterfactual, the fixed ring has leakage
\SeedSeventeenFixedLeakage{}, hotspot score \SeedSeventeenFixedHotspot{}, and
mean interception \(\mu_I=\SeedSeventeenFixedMeanIntercept{}\), interception
dispersion \(R_I=\SeedSeventeenFixedInterceptCv{}\), and certified leakage
upper bound \SeedSeventeenFixedLeakBound{}.  The tuned screen with offset
\(\Delta x=\SeedSeventeenTunedOffset{}\) and span
\(h=\SeedSeventeenTunedSpan{}\) reduces leakage to
\SeedSeventeenTunedLeakage{}, hotspot score to \SeedSeventeenTunedHotspot{},
raises mean interception to \(\SeedSeventeenTunedMeanIntercept{}\), reduces
\(R_I\) to \SeedSeventeenTunedInterceptCv{}, and tightens the certificate to
\SeedSeventeenTunedLeakBound{}.  The row-level evidence is in
\texttt{anc/results/tables/seed17\_tactical\_audit.csv}.

\paragraph{Alternative target localizers.}
The server audit compares six target-localization rules before they are turned
into screen positions: asset-only, attacker centroid, dataset prior,
prior-centroid mixture, pattern-fusion field, and posterior field.  The
posterior field uses both attacker pressure and defender coverage, whereas
pattern fusion uses attacker geometry plus the VisDrone spatial prior.  These
are deterministic policy summaries rather than bounded-rational response
models; quantal-response relaxations \citep{mckelvey1995qre} can be added on
top of the same payoff matrix, but the present audit keeps the comparison
transparent.  On
\RemoteStrategyScenarios{} dataset-conditioned samples, the posterior localizer
has mean error \RemoteStrategyPosteriorError{} against the oracle hotspot, while
pattern fusion has \RemoteStrategyFusionError{}.  However, the best
single-frame screen payoff is \RemoteStrategyBestValue{} for
\RemoteStrategyBest{}, showing that the most accurate hotspot estimate is not
necessarily the best one-step screen when exposure and geometry regularization
are included.  This motivates the repeated-game layer rather than treating
target identification as a single isolated argmax.

\paragraph{Server-side scale audit.}
The dataset-backed local run is intentionally small enough to reproduce in the
source package.  To check that the finite-game implementation scales, we also
include a CUDA batched audit for the remote server.  The server script supports
6v6, 8v8, 16v16, and 32v32 random capability tensors.  The available remote
audit reports \RemoteGpuScenarios{} synthetic games up to
\RemoteGpuMaxSwarmPair{} and solves each 4-by-4 formation game for
\RemoteGpuSteps{} multiplicative-weights iterations with batch size
\RemoteGpuBatch{}.  It reports mean value \RemoteGpuMeanValue{}, mean saddle
gap \RemoteGpuMeanGap{}, maximum saddle gap \RemoteGpuMaxGap{}, and peak
allocated memory \RemoteGpuPeakGb{} GB in \RemoteGpuSeconds{} seconds. Hardware
identifiers and live telemetry are kept in the ancillary audit files rather
than treated as a contribution.

\begin{figure}[H]
\centering
\includegraphics[width=\linewidth]{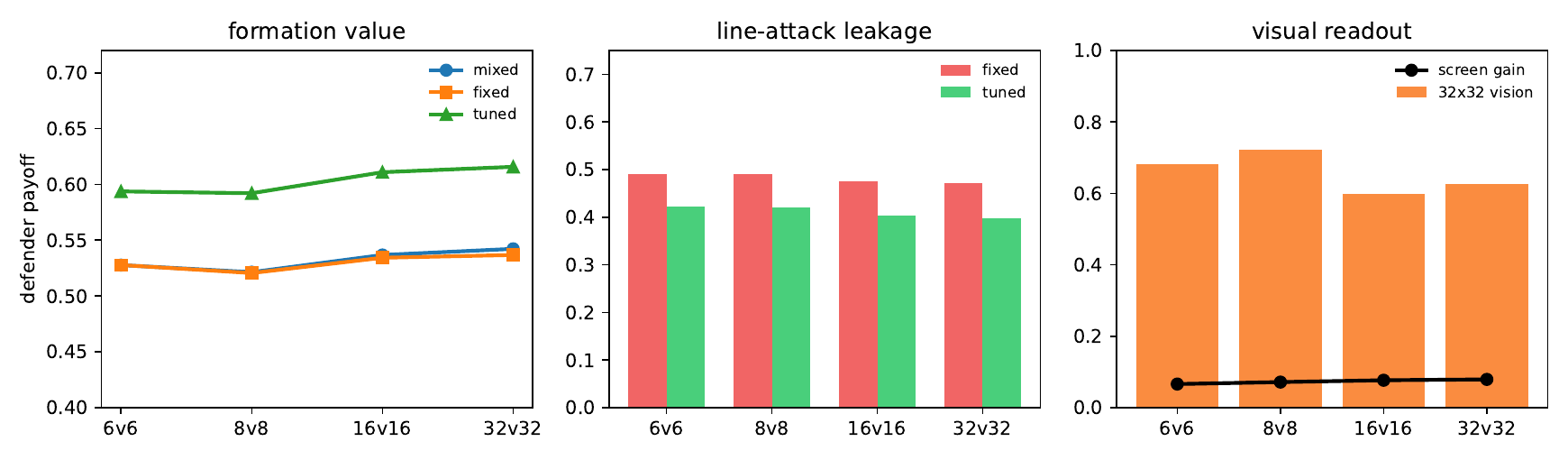}
\caption{Multi-scale formation audit from 6v6 to \MaxSwarmPair{}.  The left
panel compares mixed, fixed-screen, and tuned-screen values; the center panel
shows fixed versus tuned leakage under a line attack; the right panel reports
the 32-by-32 visual-regime readout and tuned-screen gain.  At \MaxSwarmPair{},
the tuned screen value is \ThirtyTwoTunedScreenValue{} and the visual readout
accuracy is \ThirtyTwoVisualAccuracy{}.}
\label{fig:multiscale}
\end{figure}

\section{Repeated System Game}

The single-frame game asks whether a rendered state is secure.  The larger
question is whether a virtual defensive system remains sustainable over repeated
rounds.  We model this as a finite-horizon Markov game
\citep{shapley1953stochastic,littman1994markov}.  The normalized state is
\[
s_t=(D_t,I_t,F_t,J_t,R_t,P_t,O_t,B_t,G_t,M_t,E_t)\in[0,1]^{11},
\]
where \(D_t\) and \(I_t\) are drone-pressure and interceptor-stock variables,
\(F_t\) is operator fatigue, \(J_t\) is electronic-warfare adaptation,
\(R_t\) is radar exposure, \(P_t\) is a public fear/stress index, \(O_t\) is insurance/oil
pressure, \(B_t\) is budget justification, \(G_t\) is dataset growth, \(M_t\)
is model improvement, and \(E_t\) is the entropy budget.  The entropy budget is
an information-diversity variable: policies that keep observing diverse
patterns preserve learning capacity, while overexposed or jammed states consume
it.  This follows the same spirit as entropy-regularized sequential decision
models \citep{geist2019regularized}, but here it is a transparent accounting
state rather than a learned controller.

A policy \(\pi\) is summarized by six abstract parameters:
base security \(b_\pi\), operational intensity \(\iota_\pi\), exposure
\(r_\pi\), target error \(\epsilon_\pi\), entropy production \(h_\pi\), and data
capture \(c_\pi\).  The attacker mode \(m_t\) supplies pressure \(a_t\), jamming
\(j_t\), social stress \(p_t\), cost \(k_t\), and pattern entropy \(q_t\).  The
one-step security score is
\[
\begin{aligned}
S_t^\pi=\Pi_{[0,1]}\!\big[&
b_\pi +0.16M_t+0.06\sqrt{G_t}+0.06E_t h_\pi\\
&-0.18J_tj_t-0.11F_t-0.18(0.35-I_t)_+
-0.16\epsilon_\pi a_t\big],
\end{aligned}
\]
and leakage is \(L_t^\pi=\Pi_{[0,1]}[a_tD_t(1-S_t^\pi)]\).  The key transition
coordinates are clipped coordinate-wise by \(\Pi_{[0,1]}\):
\[
\begin{aligned}
F_{t+1}&=\Pi_{[0,1]}[0.86F_t+0.11\iota_\pi k_t+0.055L_t^\pi],\\
J_{t+1}&=\Pi_{[0,1]}[0.83J_t+0.12j_t r_\pi+0.035L_t^\pi],\\
R_{t+1}&=\Pi_{[0,1]}[0.78R_t+0.17r_\pi(0.35+0.65a_t)],\\
P_{t+1}&=\Pi_{[0,1]}[0.82P_t+0.22L_t^\pi+0.06R_t],\\
O_{t+1}&=\Pi_{[0,1]}[0.85O_t+0.18P_t+0.16L_t^\pi],\\
B_{t+1}&=\Pi_{[0,1]}[0.80B_t+0.18P_t+0.14O_t+0.10G_t-0.06R_t],\\
G_{t+1}&=\Pi_{[0,1]}[0.90G_t+0.14c_\pi(0.35+q_t)(0.30+L_t^\pi+0.35R_{t+1})],\\
M_{t+1}&=\Pi_{[0,1]}[0.92M_t+0.10\log(1+4G_{t+1})E_t-0.035J_t],\\
E_{t+1}&=\Pi_{[0,1]}[0.88E_t+0.10h_\pi+0.06q_t-0.08J_t].
\end{aligned}
\]
Other stock updates follow the same accounting rule: high intensity and
attacker cost consume interceptor stock, while budget justification partially
restores it; security reduces drone pressure, while attacker pressure
replenishes it.

The repeated payoff is no longer only frame security.  It is
\[
\begin{aligned}
u_t^\pi={}&0.30S_t^\pi+0.16H_t+0.12B_{t+1}+0.12G_{t+1}
          +0.12M_{t+1}+0.08E_{t+1}\\
&-0.10F_{t+1}-0.08J_{t+1}-0.07R_{t+1}
-0.08P_{t+1}-0.07O_{t+1}-0.04(0.35-I_{t+1})_+,
\end{aligned}
\]
where \(H_t=0.55(D_{t+1}+I_{t+1})/2+0.25M_{t+1}+0.20E_{t+1}\) is a
system-health term.  The evaluated value is
\[
V^\pi=\mathbb{E}\left[\frac{\sum_{t=0}^{T-1}\gamma^t u_t^\pi}
{\sum_{t=0}^{T-1}\gamma^t}\right], \qquad \gamma=0.96.
\]

\begin{proposition}[Bounded repeated-game estimator]
For any fixed policy \(\pi\), attacker-mode distribution, and horizon \(T\),
the projected transition map keeps all coordinates in \([0,1]\), and the
discount-normalized trajectory return \(Y^\pi\) lies in \([-1,1]\).  Therefore
the Monte Carlo estimator \(\widehat V_N^\pi=N^{-1}\sum_{n=1}^NY_n^\pi\) obeys
\[
\Pr\{|\widehat V_N^\pi-V^\pi|\ge \eta\}\le
2\exp\!\left(-\frac{N\eta^2}{2}\right).
\]
\end{proposition}

\begin{proof}
Each transition coordinate is an affine or monotone nonlinear expression
followed by projection onto \([0,1]\), so the unit hypercube is invariant.
The implementation clips one-step utility to \([-1,1]\); a convex
discount-normalized average of such utilities also lies in \([-1,1]\).
Applying Hoeffding's bounded-sum inequality
\citep[Theorem~1]{hoeffding1963probability} to the independent trajectory
returns \(Y_1^\pi,\ldots,Y_N^\pi\in[-1,1]\) gives
\[
  \Pr\{|\widehat V_N^\pi-V^\pi|\ge \eta\}
  \le 2\exp\!\left(-\frac{2N^2\eta^2}{\sum_{n=1}^{N}(1-(-1))^2}\right)
  =2\exp\!\left(-\frac{N\eta^2}{2}\right).
\]
\end{proof}

The local audit uses \MarkovPolicyCount{} candidate policies for
\MarkovHorizon{} rounds.  The best local policy is \MarkovBestPolicy{} with
value \MarkovBestValue{}, survival \MarkovBestSurvival{}, final entropy budget
\MarkovBestEntropy{}, dataset growth \MarkovBestData{}, and model improvement
\MarkovBestModel{}.

The server audit conditions policy parameters on the full-dataset strategy
audit and evaluates \RemoteMarkovTrajectories{} trajectories for
\RemoteMarkovHorizon{} rounds.  The best long-horizon policy is
\RemoteMarkovBestPolicy{} with value \RemoteMarkovBestValue{} and gain
\RemoteMarkovGain{} over fixed screen.  This result is structurally different
from the single-frame result: the best one-step screen is
\RemoteStrategyBest{}, while the best repeated policy is budget-adaptive,
because it preserves entropy budget and model improvement while limiting
fatigue, exposure, and insurance/oil pressure.

\paragraph{Ethics stress audit.}
The repeated game is also scored under four non-operational risk dimensions:
false-alarm burden, overcollection risk, escalation pressure, and autonomy-risk
pressure.  The dimensions are motivated by the risk-management language of
NIST AI RMF, the OECD AI Principles, and the Russian AI ethics code
\citep{nist2023airmf,oecd2019aiprinciples,russian2021aicode}.  They are not a
moral verdict and do not authorize deployment; they are a reproducibility check
that a policy with higher game value is not silently selected by pushing all
social cost into exposure, fear/stress, data capture, or model autonomy.  The
ancillary table \texttt{ethics\_stress\_audit.csv} reports the four risk scores
and a simple risk-adjusted value.  Under this penalty the best local policy is
\EthicsRiskAdjustedBest{} with risk-adjusted value \EthicsRiskAdjustedValue{}.

\begin{figure}[H]
\centering
\includegraphics[width=\linewidth]{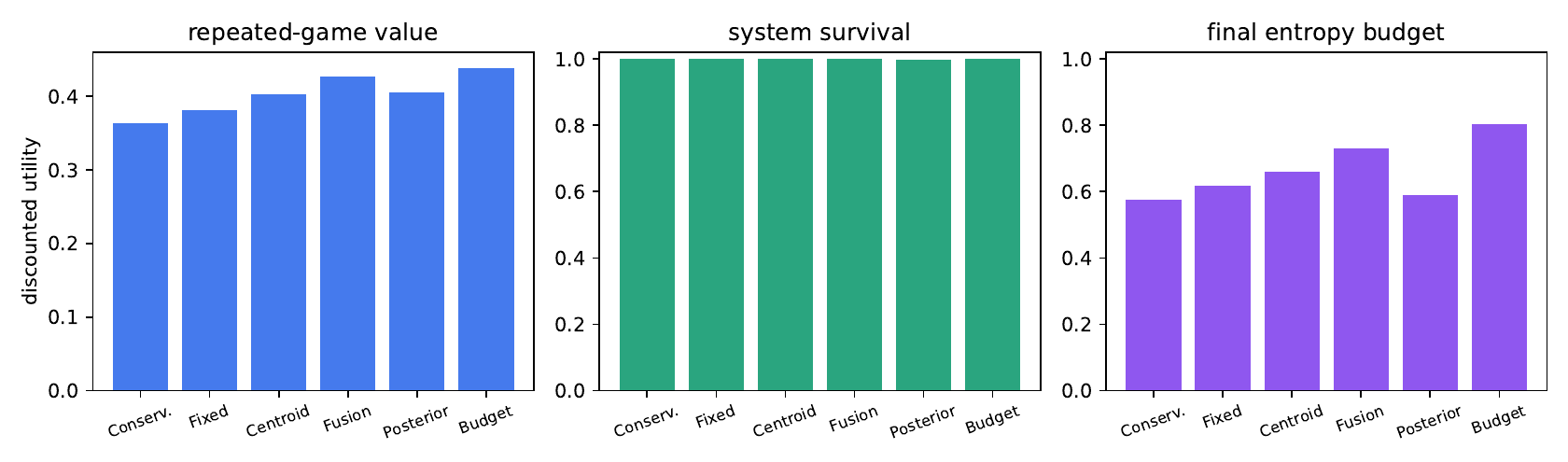}
\caption{Repeated system-game audit.  The bars compare candidate policy values,
survival probabilities, and final entropy budgets in the local Markov
simulation.  The server-side audit repeats the same state model at larger scale
with policy parameters conditioned on full-dataset strategy results.}
\label{fig:markov}
\end{figure}

\section{Visual Pattern Readout}

The game state is rendered as a tactical image rather than only as a matrix.
A real VisDrone aerial image is dimmed and used as background.  Virtual
defender and attacker positions are overlaid with nearest-neighbor distance
links, confidence scores, an asset marker, and a prediction banner.  This makes
the output visually inspectable and also supplies a raster for pattern
recognition.

\begin{figure}[H]
\centering
\includegraphics[width=\linewidth]{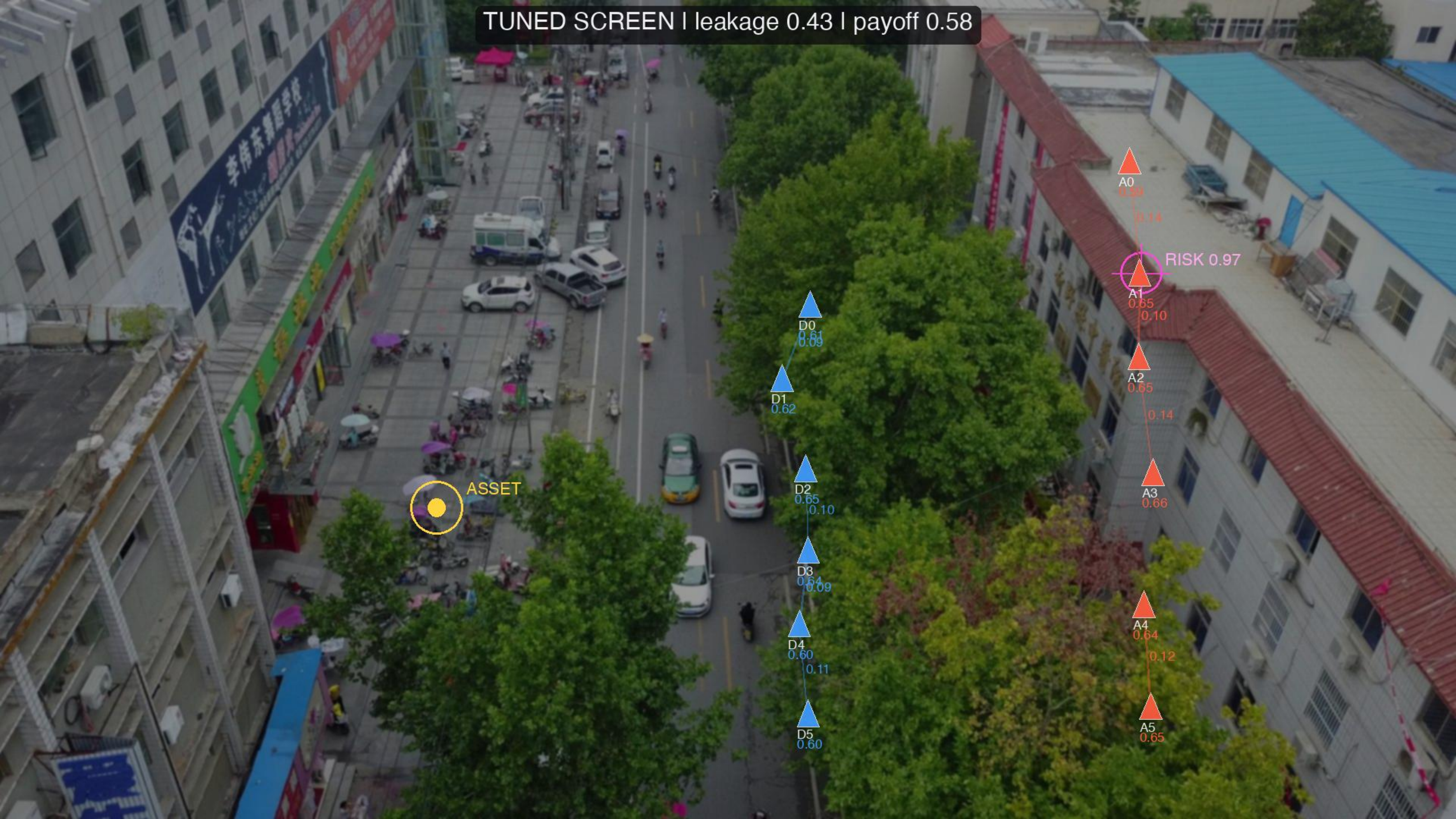}
\caption{Human-readable tactical overlay for the seed-17 scenario on a real
VisDrone frame.  The triangles are virtual agents generated from dataset-conditioned
capability vectors; nearest-neighbor links, confidence scores, asset marker,
and hotspot marker are rendered from the same seed.  For this exact state the
tuned screen reports leakage \SeedSeventeenTunedLeakage{}, payoff
\DemoValue{}, hotspot score \SeedSeventeenTunedHotspot{}, and interception
dispersion \(R_I=\SeedSeventeenTunedInterceptCv{}\).}
\label{fig:overlay}
\end{figure}

\begin{figure}[H]
\centering
\includegraphics[width=\linewidth]{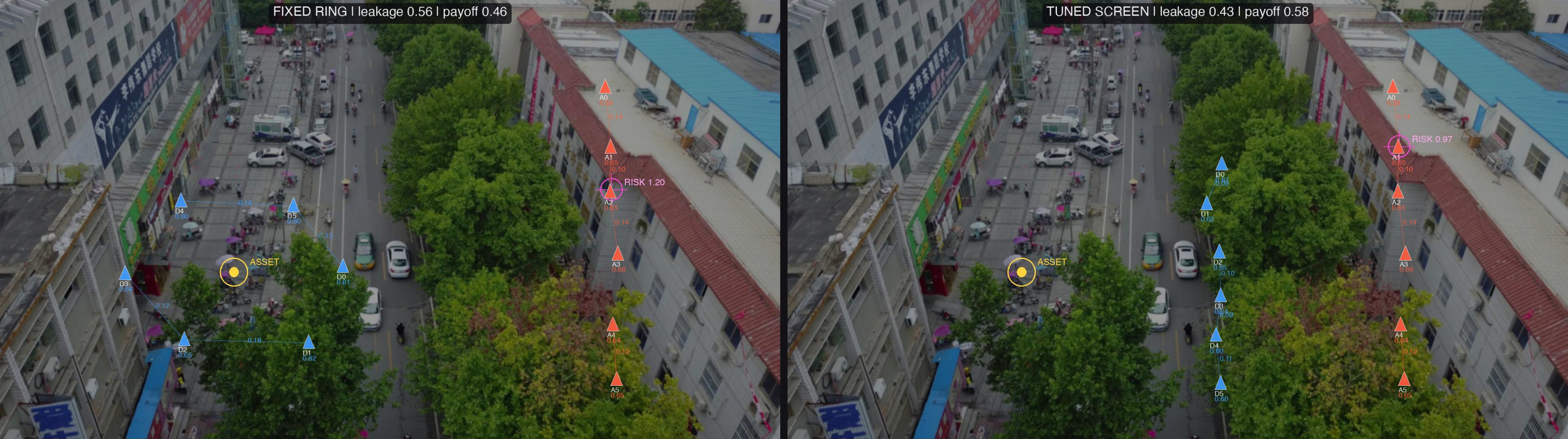}
\caption{Real-field visual counterfactual on the same VisDrone frame.  The left
panel shows a fixed-ring defender against the same line attack; the right panel
shows the tuned screen response for seed 17.  The fixed ring has leakage
\SeedSeventeenFixedLeakage{}, hotspot score \SeedSeventeenFixedHotspot{}, and
\(R_I=\SeedSeventeenFixedInterceptCv{}\).  The tuned screen moves to
\(\Delta x=\SeedSeventeenTunedOffset{}\), \(h=\SeedSeventeenTunedSpan{}\), and
reduces those values to \SeedSeventeenTunedLeakage{},
\SeedSeventeenTunedHotspot{}, and \SeedSeventeenTunedInterceptCv{},
respectively.}
\label{fig:before-after}
\end{figure}

\begin{figure}[H]
\centering
\includegraphics[width=\linewidth]{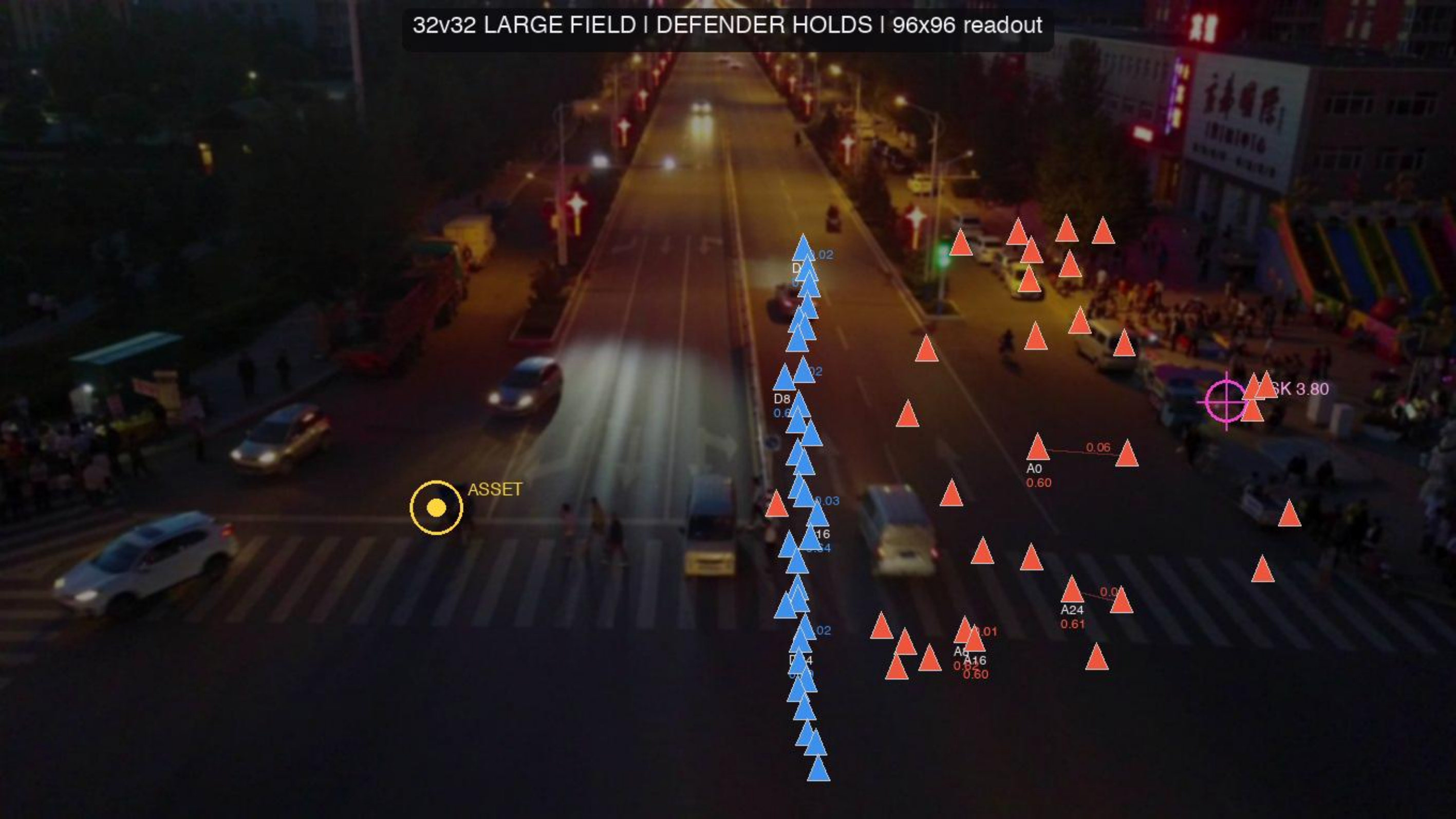}
\caption{Large-field 32v32 visual state on a different real VisDrone frame.  The
renderer keeps all agents virtual but expands the field to a denser 32v32 state;
labels and nearest-neighbor links are subsampled for legibility.  The readout
identifies the state as defender-holds with leakage \LargeFieldLeakage{}, payoff
\LargeFieldValue{}, and hotspot score \LargeFieldHotspot{}.}
\label{fig:large-field}
\end{figure}

For a seed \(s\), the overlay is the visible form of a certificate tuple
\[
  \mathcal C_s=
  \left(L_s,u_s,\mu_{I,s},R_{I,s},
  \bar w_s\!\left[1-\mu_{I,s}(1-\sqrt{n_a-1}R_{I,s})\right],
  \widehat z_{m,s},\widehat y_{m,s}\right),
\]
where \(L_s\) is leakage, \(u_s\) is payoff, \(\mu_I\) and \(R_I\) are the mean
and coefficient of variation of abstract interception scores, the fifth entry
is the equalization leakage certificate, \(\widehat z_{m,s}\) is the raster
hotspot, and \(\widehat y_{m,s}\) is the visual-regime label.  Figure
\ref{fig:large-field} therefore renders the same objects that enter the
finite-game payoff, leakage proof, and visual-regime audit; the figure is not a
separate illustrative layer.

To test whether the rendered state carries recoverable structure, we rasterize
defender positions, attacker positions, and the asset marker into a \(32\times
32\times3\) occupancy image.  Payoffs are split into three empirical regimes:
attacker-leaks, contested, and defender-holds.  A nearest-centroid classifier
trained on the raster alone reaches \CvAccuracy{} accuracy on \CvTestN{}
held-out 6v6 states.  The same 32-by-32 readout is repeated in the scale audit;
at \MaxSwarmPair{} it reaches \ThirtyTwoVisualAccuracy{} on
\ThirtyTwoVisualTestN{} held-out states.  This is not intended as a strong
classifier; it is a leakage check that the visual state contains enough
geometric signal to recover the strategic regime above chance and that the
signal survives larger swarms.

\paragraph{Field-resolution audit.}
The 32-by-32 raster is not assumed to be optimal.  For a rendered field
resolution \(m\in\{32,48,64,96,128\}\), let
\[
  \Phi_m(x,y,a_0)=\big(\phi^d_m,\phi^a_m,\phi^0_m\big)\in\mathbb R^{3m^2}
\]
be the Gaussian occupancy raster for defender positions, attacker positions, and
the asset marker.  A raster hotspot is computed from
\[
  \widehat H_m(z)=\phi^a_m(z)+0.45\phi^0_m(z)-0.85\phi^d_m(z),
\]
and is compared with the continuous hotspot \(H(z)\).  The local audit uses 384
32v32 states and reports held-out regime accuracy, mean hotspot error, active
cell ratio, and normalized map entropy.  Accuracy peaks at
\FieldResolutionBestSide{}-by-\FieldResolutionBestSide{} with
\FieldResolutionBestAccuracy{} under the fixed-pixel encoder, while the
128-by-128 raster falls to \FieldResolutionOneTwentyEightAccuracy{} and has mean
hotspot error \FieldResolutionOneTwentyEightError{} rather than the 32-by-32
error \FieldResolutionThirtyTwoError{}.  This mismatch diagnoses a bandwidth
error: fixed-pixel Gaussians shrink in image-normalized coordinates as \(m\)
increases, producing sparse maps and lower entropy
(\FieldResolutionOneTwentyEightEntropy{} at 128-by-128).

We therefore evaluate the finite encoder family in the scale-normalized raster
risk proposition.  At 128-by-128 the selected encoder is
\FieldOptimizedOneTwentyEightEncoder{} with
\(\lambda=\FieldOptimizedOneTwentyEightLambda{}\); it reaches
\FieldOptimizedOneTwentyEightAccuracy{} accuracy, reduces mean hotspot error to
\FieldOptimizedOneTwentyEightError{}, improves classification loss by
\FieldOptimizedOneTwentyEightLossGain{}, improves hotspot error by
\FieldOptimizedOneTwentyEightHotspotGain{}, and reduces the joint loss to
\FieldOptimizedOneTwentyEightJointLoss{} with gain
\FieldOptimizedOneTwentyEightJointGain{}.  A server-side GPU audit repeats the
occupancy analysis over \RemoteFieldScenarios{} synthetic 32-agent states up to
\RemoteFieldMaxSide{}-by-\RemoteFieldMaxSide{} in \RemoteFieldSeconds{} seconds
and records peak allocation \RemoteFieldPeakGb{} GB; it selects the same
scale-normalized family at 128-by-128 and reduces hotspot error from the
fixed-pixel regime to the optimized readout.

\begin{figure}[H]
\centering
\includegraphics[width=\linewidth]{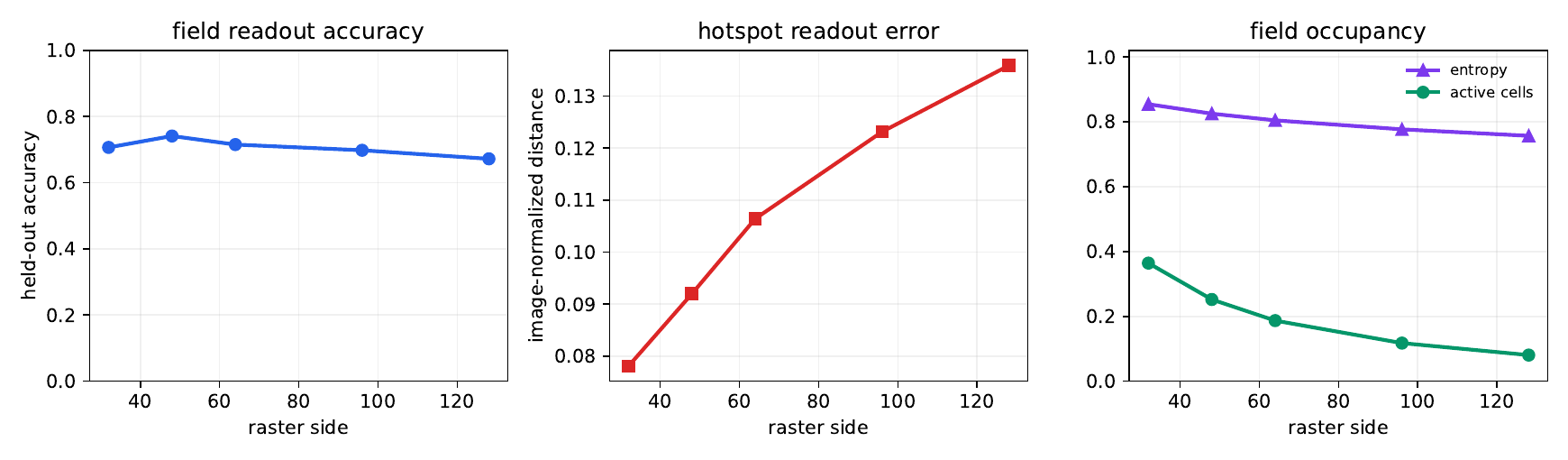}
\caption{Field-resolution audit from 32-by-32 to
\FieldResolutionMaxSide{}-by-\FieldResolutionMaxSide{} rasters.  The panels
compare visual-regime accuracy, raster-hotspot error, and occupancy
entropy/active-cell behavior.  The result links the visual benchmark to the
sparsity and hotspot-stability lemmas: larger rendered fields add numerical
detail, but fixed-pixel kernels can shrink the image-plane support and reduce
the classifier signal.}
\label{fig:field-resolution}
\end{figure}

\begin{figure}[H]
\centering
\includegraphics[width=\linewidth]{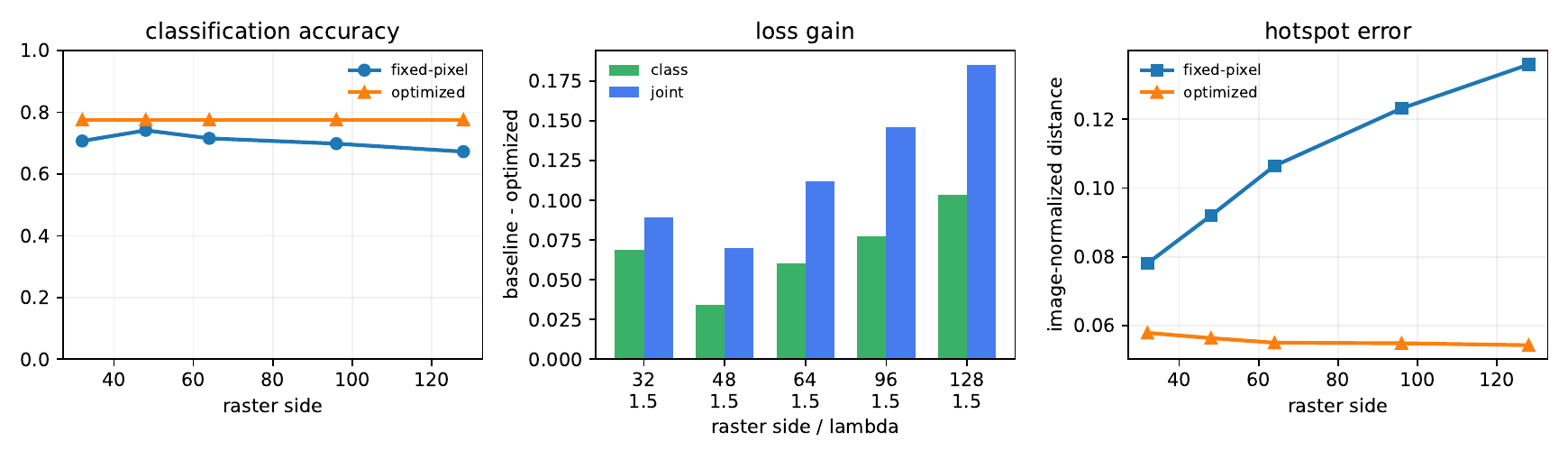}
\caption{Risk-minimized visual encoder for 32-by-32 through
\FieldResolutionMaxSide{}-by-\FieldResolutionMaxSide{} rasters.  The
optimization compares the fixed-pixel baseline with scale-normalized Gaussian
bandwidths and selects the encoder minimizing
\((1-\widehat a_m)+\widehat e_m\).  The 128-by-128 result is the numerical
instance of the finite encoder selection bound: the selected encoder cannot
increase empirical joint loss over the included fixed-pixel baseline, and here
both classification loss and hotspot error improve.}
\label{fig:field-optimization}
\end{figure}

\begin{figure}[H]
\centering
\includegraphics[width=\linewidth]{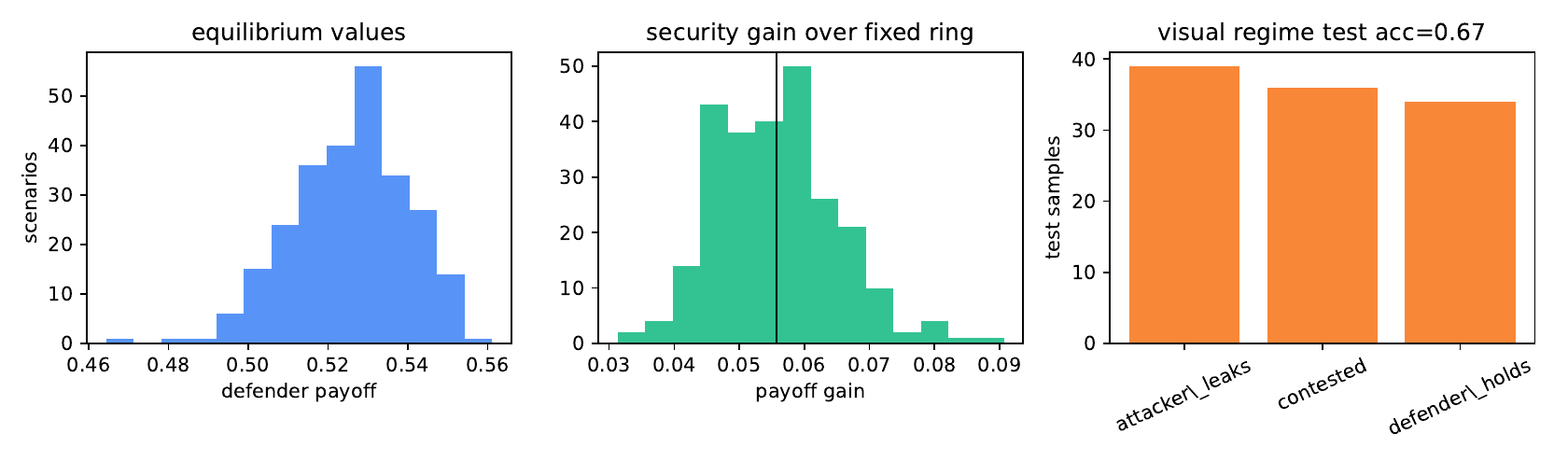}
\caption{Experiment summary.  The center panel reports the adversarial security
gain of the mixed policy over a fixed-ring defender.  The right panel reports
the held-out visual-regime benchmark.}
\label{fig:summary}
\end{figure}

\section{Reproducible Environment}

The repository separates the article source from reproducibility material.
The article source contains \texttt{main.tex}, \texttt{refs.bib},
\texttt{numbers.tex}, and selected figures.  Reproducibility material contains
code, scripts, result CSVs, JSON summaries, dataset manifests, Docker
instructions, formal checks, and server download scripts.  Large raw datasets
are not shipped with the source bundle.

The ancillary formal checks are deliberately narrow.  Lean and Coq files verify
the finite encoder-selection inequality used by Figure~\ref{fig:field-optimization}
and the fixed-point 128-by-128 numeric gain identities recorded in
\texttt{field\_numeric\_certificates.csv}.  They do not attempt to formalize the
whole computer-vision pipeline; they cover the representation-risk step where a
small algebraic or numeric mismatch would directly change the paper's field
resolution conclusion.

The local command sequence is:
\begin{quote}
\small
\begin{verbatim}
make experiments
make paper
make bundle
\end{verbatim}
\end{quote}
For an isolated environment, the repository also includes a Dockerfile:
\begin{quote}
\small
\begin{verbatim}
docker build -t visual-swarm-games .
docker run --rm -v "$PWD:/work" visual-swarm-games \
  make experiments paper
\end{verbatim}
\end{quote}
The server-side scripts download the full VisDrone and UAVSwarm material,
Anti-UAV reference repositories, and Sheffield provenance files into
\texttt{/workspace/data} when a remote machine is available.  A script flag
enables the much larger Sheffield video archives.  The same server bundle also
includes a CUDA batched game-theory audit that can run large
multiplicative-weights sweeps on the remote GPU without placing large data
files in the local repository.

\section{Limitations}

The current payoff model is intentionally compact.  It is a controlled
benchmark, not an aerodynamic or operational simulator.  The Bloom-filter
capability draw is useful because it is deterministic, compact, and tied to
real annotation statistics, but it is not a learned physical latent model.  The
visual classifier is deliberately weak; stronger recognition models should be
evaluated only after larger server-side datasets are materialized and split
with source-aware leakage controls.  Finally, the formation game uses a small
action set so that the matrix remains inspectable.  This is a virtue for an
auditable reproducibility bundle, but not a claim of tactical completeness.

\section{Conclusion}

We introduced a pattern-derived virtual swarm benchmark in which real
drone-vision annotations become a Bloom-filter sketch, the sketch generates
synthetic capability vectors, and those vectors populate multi-scale formation
games up to \MaxSwarmPair{}.  The same state is rendered as a human-readable
image and as a 32-by-32 raster for regime classification.  The result connects
computer vision, pattern recognition, and finite/repeated game analysis while
keeping the operational boundary explicit: the artifact studies virtual states,
not real drone control.

\bibliographystyle{unsrtnat}
\bibliography{refs}

\end{document}